%% file: main.tex
\documentclass[letterpaper]{article} 
\usepackage{aaai23}  
\usepackage{times}  
\usepackage{helvet}  
\usepackage{courier}  
\usepackage[hyphens]{url}  
\usepackage{graphicx} 
\usepackage{natbib}  
\usepackage{caption} 
\usepackage{algorithm}
\usepackage{algorithmic}

\usepackage{subcaption}
\usepackage{graphicx}
\usepackage{microtype}

\usepackage{amsmath}
\usepackage{amssymb}
\usepackage{mathtools}
\usepackage{amsthm}

\usepackage{newfloat}
\usepackage{listings}
\DeclareCaptionStyle{ruled}{labelfont=normalfont,labelsep=colon,strut=off} 
\floatstyle{ruled}
\newfloat{listing}{tb}{lst}{}
\floatname{listing}{Listing}
\title{Variable-Based Calibration for Machine Learning Classifiers}
\author{
    Markelle Kelly, Padhraic Smyth
}
\affiliations{
    University of California, Irvine

    kmarke@uci.edu, smyth@ics.uci.edu
}

\theoremstyle{plain}
\newtheorem{theorem}{Theorem}[section]

\theoremstyle{definition}
\newtheorem{definition}[theorem]{Definition}

\theoremstyle{remark}

\begin{document}

\maketitle

\begin{abstract}
The deployment of machine learning classifiers in high-stakes domains requires well-calibrated confidence scores for model predictions. In this paper we introduce the notion of variable-based calibration to characterize calibration properties of a model with respect to a variable of interest, generalizing traditional score-based metrics such as expected calibration error (ECE). In particular, we find that models with near-perfect ECE can exhibit significant miscalibration as a function of features of the data. We demonstrate this phenomenon both theoretically and in practice on multiple well-known datasets, and show that it can persist after the application of existing calibration methods. To mitigate this issue, we propose strategies for detection, visualization, and quantification of variable-based calibration error. We then examine the limitations of current score-based calibration methods and explore potential modifications. Finally, we discuss the implications of these findings, emphasizing that an understanding of calibration beyond simple aggregate measures is crucial for endeavors such as fairness and model interpretability. 
\end{abstract}

\section{Introduction}

Predictive models built by machine learning algorithms are increasingly informing decisions across high-stakes applications such as medicine \citep{rajkomar2019machine}, employment \citep{chalfin2016productivity}, and criminal justice \citep{zavrsnik2021}. There is also broad recent interest in developing systems where humans and machine learning models collaborate to make predictions and decisions \citep{kleinberg2018human,bansal2021most,de2021classification,steyvers2022bayesian}. A critical aspect of using model predictions in such contexts is calibration. In particular, in order to trust the predictions from a machine learning classifier, these predictions must be accompanied by well-calibrated confidence scores. 

In practice, however, it has been well-documented that machine learning classifiers such as deep neural networks can produce poorly-calibrated class probabilities \citep{guo2017calibration,vaicenavicius2019evaluating,ovadia2019}. As a result, a variety of calibration methods have been developed, which aim to ensure that a model's confidence (or score) matches its true accuracy. A widely used approach is post-hoc calibration: methods which use a separate labeled dataset to learn a mapping from the original model's class probabilities to calibrated probabilities, often with a relatively simple one-dimensional mapping (e.g., \citet{plattprobabilistic,kull2017,kumar2019verified}). These methods have been shown to generally improve the the empirical calibration error of a model, as commonly measured by the expected calibration error (ECE).

However, as we show in this paper, aggregate measures of score-based calibration error such as ECE can hide significant systematic miscalibration in other dimensions of a model's performance. To address this issue we introduce the notion of {\it variable-based calibration} to better understand how the calibration error of a model can vary as a function of a variable of interest, such as an input variable to the model or some other metadata variable. We focus in particular in this paper on real-valued variables. For example, in prediction problems involving individuals (e.g., credit-scoring or medical diagnosis) one such variable could be {\it Age}. Detecting systematic miscalibration is important for problems such as assessing the fairness of a model,
for instance detecting that a model is significantly overconfident for some age ranges and underconfident for others. 

As an illustrative example, consider a simple classifier trained to predict the presence of cardiovascular disease\footnote{https://www.kaggle.com/sulianova/cardiovascular-disease-dataset}.
After the application of Platt scaling, a standard post-hoc calibration method, this model attains a relatively low ECE of 0.74\%. This low ECE is reflected in the reliability diagram shown in Figure \ref{fig:introa}, which shows near-perfect alignment with the diagonal. If a user of this model were to only consider aggregate metrics such as ECE, they might reasonably conclude that the model is generally well-calibrated. However, evaluating model error and predicted error with respect to the variable {\it Patient Age} reveals an undesirable and systematic miscalibration pattern with respect to this variable, as illustrated in Figure \ref{fig:introb}: the model is underconfident by upwards of five percentage points for younger patients, and is significantly overconfident for older patients.

\begin{figure*}
    \centering
    \begin{subfigure}[b]{\columnwidth}
        \centering
        \includegraphics[width=\linewidth]{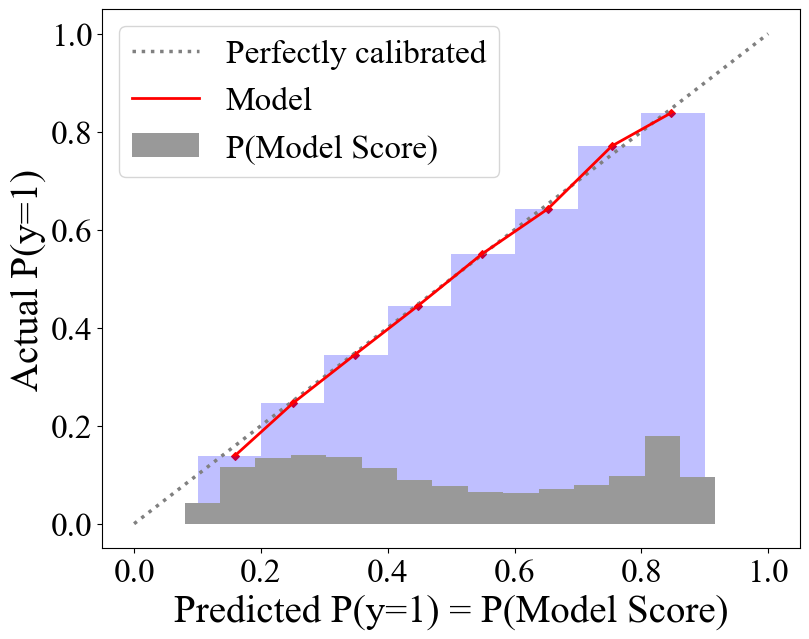} 
        \caption{Reliability diagram (for accuracy)} 
        \label{fig:introa}
    \end{subfigure}
    \begin{subfigure}[b]{\columnwidth}
        \centering
        \includegraphics[width=\linewidth]{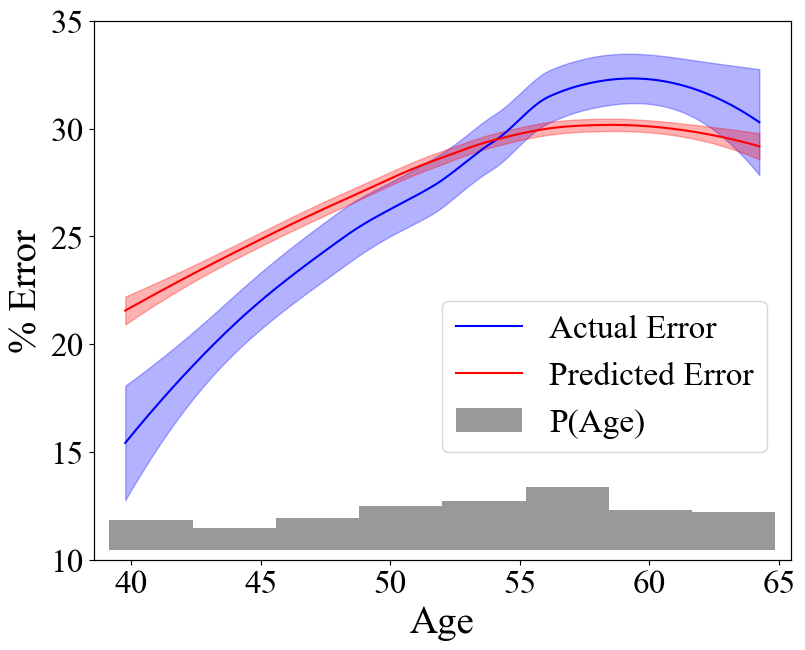} 
        \caption{Variable-based calibration plot (for error)} \label{fig:introb}
    \end{subfigure}
    \caption{Calibration plots for a neural network predicting cardiovascular disease, after calibration with Platt scaling: (a) reliability diagram, (b) LOESS-smoothed estimates with confidence intervals of actual and model-predicted error as a function of patient age. This dataset consists of 70,000 records of patient data (49,000 train, 6,000 validation, 15,000 test), with a binary prediction task of determining the presence of cardiovascular disease.}
    \label{fig:intro}
\end{figure*}

In this paper, we systematically investigate variable-based calibration for classification models, from both theoretical and empirical perspectives. 
In particular, our contributions are as follows: 
\begin{enumerate}
    \item We introduce the notion of {\it variable-based calibration} and define a per-variable calibration metric (VECE).
    \item We characterize theoretically the relationship between variable-based miscalibration measured via VECE and traditional score-based miscalibration measured via ECE.  
    \item We demonstrate, across multiple well-known tabular, text, and image datasets and a variety of models, that significant variable-based miscalibration can exist in practice, even after the application of standard score-based calibration methods.
    \item We investigate \textit{variable-based calibration methods} and demonstrate empirically that 
    these methods can simultaneously reduce both ECE and VECE.
    \footnote{\raggedright{Our code is available online at https://github.com/markellekelly/variable-wise-calibration.}}
\end{enumerate}

\section{Related Work}
\label{section:related_work}

\paragraph{Visualizing Model Performance by Variable:} 
In prior work a number of different techniques have been developed for visual understanding and diagnosis of model performance with respect to a particular variable of interest. One such technique is partial dependence plots \citep{friedman2001greedy, molnar2020interpretable}, which visualize the effect of an input feature of interest on model predictions. Another approach is dashboards such as FairVis \citep{cabrera2019fairvis} which enable the exploration of model performance (e.g., accuracy, false positive rate) across various data subgroups. However, 
none of this prior work investigates the visualization of per-variable calibration properties of a model,
i.e., how a model's own predictions of accuracy (or error) vary as a function of a particular variable.

\paragraph{Quantifying Model Calibration by Variable:} 
Work on calibration for machine learning classifiers has largely focused on score-based calibration: reliability diagrams, the ECE, and standard calibration methods are all defined with respect to confidence scores \citep{murphy1977reliability, huang2020tutorial, song2021classifier}. An exception to this is in the fairness literature, where researchers have broadly called for disaggregated model evaluation, e.g. computing metrics of interest individually for sensitive sub-populations \citep{mitchell2019model, raji2020closing}. To this end, several notions of calibration that move beyond standard aggregate measures have been introduced: \citet{johnsonmasses} check calibration across all identifiable subpopulations of the data, \citet{pan2020field} evaluate calibration over data subsets corresponding to a categorical variable of interest, and \citet{luolocalized} compute ``local calibration'' using the average classification error on similar samples. 
Our paper expands on this prior work in two ways. First, we shift the focus from categorical to real-valued variables---our methods operate on a continuous basis, estimating calibration for an entire population rather than for various subgroups. Second, we center on diagnosing calibration; we present visualization and estimation techniques for understanding an existing classifier rather than prescriptive conditions for model training or selection.

\section{Background on Score-Based ECE}
\label{section:background} 
Consider a classification problem mapping inputs $x$ to predictions for labels  $y \in \{1, \ldots, K\}$. Let $f$ be a black-box classifier which outputs label probabilities $f(x) \in [0,1]^K$ for each $x \in X$. Then, for the standard 0-1 loss function, the predicted label is $\hat{y} = \text{argmax}(f(x)) \in \{1, \ldots, K\}$ and the corresponding confidence score is $s = s(x) = P_f(y=\hat{y} | x) = \text{max}(f(x))$. It is of interest to determine whether such a model is \textit{well-calibrated}, that is, whether its confidence matches the true probability that a prediction is correct. 

For a given confidence score $s$, we define $\text{Acc}(s) = P(y=\hat{y}|s) = \mathbb{E} \left[ \mathbb{I}[y=\hat{y} |s] \right]$.
Then the $\ell_p$ calibration error (CE), as a function of the confidence score $s$, is defined as the difference between accuracy and confidence score  \citep{kumar2019verified}:
\begin{equation}
\label{eqn:ce}
    \text{CE}(s) =  | P(y = \hat{y}|s) - s |^p = | \text{Acc}(s) - s |^p
\end{equation}
where $p\geq 1$. In this paper, we will focus on the expectation of the $\ell_1$ calibration error with $p=1$, known as the ECE:
\begin{equation}
\label{eqn:ece}
    \text{ECE}  = \mathbb{E}[\text{CE}(s)]  = \int_s P(s) |\text{Acc}(s) - s| ds  
\end{equation} 
where an ECE of zero corresponds to perfect calibration. In practice, ECE is often estimated empirically on a labeled test dataset by creating $B$ bins over $s$ according to some binning scheme (e.g., \citet{guo2017calibration}):
\begin{equation}
\label{eq:ece_est}
   \widehat{\text{ECE}} = \sum_{b=1}^{B} \frac{n_b}{n} |\text{Acc}_b - \text{Conf}_b|
\end{equation}
where $n_b$ is the number of datapoints in bin $b$, $n$ is the total number of datapoints, and $\text{Acc}_b$ and $\text{Conf}_b$ are the estimated accuracy and estimated average value of confidence, respectively, in bin $b=1,\ldots,B$.

\section{Variable-Based Calibration Error}
\label{section:VCE} 
 In many applications, we may be motivated to understand the calibration properties of a classification model $f$ relative to one or more particular variables of interest.
 For instance, traditional reliability diagrams and the ECE measure may be insufficient to fully characterize the type of variable-based miscalibration shown in Figure \ref{fig:intro}.
 
Consider a real-valued variable $V$ taking values $v$. $V$ could be a variable related to the inputs $X$ of the model, such as one of the input features, another feature (e.g., metadata) defined per instance but not used in the model, or some function of inputs $x$. To evaluate model calibration with respect to $V$, we introduce the notion of \textit{variable-based calibration error} (VCE), defined pointwise as a function of $v$:
\begin{equation}
\label{eqn:vce}
    \mbox{VCE}(v) = \bigl| \text{Acc}(v) - \mathbb{E}[s|v] \bigr|
\end{equation}
where $ \text{Acc}(v) = P(y=\hat{y}|v) $ is the accuracy of the model conditioned on $V=v$, marginalizing over inputs to the model that do not involve $V$.
$\mathbb{E}[s|v]$ is the expected model score conditioned on a particular value $v$:
\begin{equation}
   \mathbb{E}[s|v] = \int_s  s\cdot P(s|v) ds
\end{equation}
In general, conditioning on $v$ will induce a distribution over inputs $x$, which in turn induces a distribution $P(s|v)$ over scores $s$ and predictions $\hat{y}$. As an example of $\text{VCE}(v)$, in the context of Figure \ref{fig:introb}, at $v = 45$, the model accuracy $P(y=\hat{y}|v)$ is estimated to be $100-21=79\%$ and the expected score $\mathbb{E}[s|v]$ is estimated to be $76\%$, so the $\mbox{VCE}(v)$ is approximately $3\%$.

The expected value of $\text{VCE}(v)$,  with respect to $V$, is defined as:
\begin{equation}
\label{eqn:vece} 
  \text{VECE} = \mathbb{E}[\text{VCE}(v)] = \int_v P(v) \text{VCE}(v) dv
\end{equation} 
\paragraph{Comment} Note that CE (and ECE) can be seen as a special case of VCE (and VECE) given the correspondence of Equations \ref{eqn:ce} and \ref{eqn:ece} with Equations \ref{eqn:vce} and \ref{eqn:vece} when $V$ is the model score (i.e., $V=s$). In the rest of the paper, however, we view CE and ECE as being distinct from VCE and VECE in order to highlight the differences between score-based and variable-based calibration.

As with ECE, a practical way to compute an empirical estimate of VECE is by binning, where bins $b$ are  defined by some binning scheme (e.g., equal weight) over values $v$ of the variable $V$ (rather than over scores $s$):
\begin{equation}
    \widehat{\text{VECE}} = \sum_{b=1}^{B} \frac{n_{b}}{n} |\text{Acc}_{b} - \text{Conf}_{b}|.
\end{equation}
Here $b$ is a bin corresponding to some sub-range of $V$, $n_{b}$ is the number of points within this bin, and $\text{Acc}_{b}$ and $\text{Conf}_{b}$ are empirical estimates of the model's accuracy and the model's average confidence within bin $b$. For example, the $\widehat{\text{ECE}}$ in Figure 1 is 0.74\%, while the $\widehat{\text{VECE}}$ is 2.04\%.

The definitions of $\text{VCE}(v)$ and $\text{VECE}$ above are in terms of a continuous variable $V$, which is our primary focus in this paper. In general, the definitions above and the theoretical results in Section \ref{section:score_loss} also apply to discrete-valued $V$, as well as to multivariate $V$.

\section{Theoretical Results}
\label{section:score_loss}

In this section, we establish a number of results 
on the relationship between ECE and VECE. All proofs can be found in Appendix A. 

First, we show that the ECE and VECE can differ by a gap of up to 50\%. 
\begin{theorem}[VECE bound]
There exist $K$-ary classifiers $f$ and variables $V$ such that the classifier $f$ has both ECE = 0 and variable-based
$ \text{VECE} = 0.5 - \frac{1}{2K}$.
\label{thm:vece}
\end{theorem} 
For example, in the binary case with $K=2$,  the difference between ECE and VECE can be as large as 0.25. As the number of classes $K$ grows, this gap approaches 0.5. Thus, we can have models $f$ that are perfectly calibrated according to ECE (i.e. with ECE = 0) but that can have VECE ranging from 0.25 to 0.5. We will show later in Section \ref{section:diffeqs} that this type of gap is not just a theoretical artifact but also exists in real-world datasets, for real-world classifiers $f$ and for specific variables $V$ of interest.
The proof of Theorem \ref{thm:vece} is by construction, using a model $f$ that is very underconfident for certain regions of $v$ and very overconfident in other regions of $v$, but perfectly calibrated with respect to $s$. 

In the context of analyzing properties of ECE, \citet{kumar2019verified} proved that the binned empirical estimator $\widehat{\text{ECE}}$  consistently underestimates the true ECE, and showed by construction that this gap can approach 0.5. Our results complement this work in that we are concerned with the true theoretical relationship between two different measures of calibration, namely ECE and VECE, whereas \citet{kumar2019verified} relate the estimate $\widehat{\text{ECE}}$ (Equation \ref{eq:ece_est}) with the true ECE (Equation \ref{eqn:ece}).

\begin{theorem}[ECE bound]
There exist K-ary classifiers $f$ and variables $V$ such that the classifier $f$ has $\text{VECE}=0$ and $\text{ECE}=0.5 - \frac{1}{2K}$.
\end{theorem}
We prove this by construction, where $f$ is well-calibrated with respect to a variable $V$, but its low scores are very underconfident and its high scores are very overconfident.

The results above illustrate that the ECE and VECE measures can be very different for the same model $f$. In our experimental results we will also show that it is not uncommon (particularly for uncalibrated models) for ECE and VECE to be equal.
To understand the case of equality, we first define the notion of {\it consistent over- or under-confidence} with respect to a variable:
\begin{definition}[Consistent overconfidence]
\label{sec:overconf}
Let $f$ be a classifier with scores $s$. For a variable $V$ taking values $v$,  $f$ is {\it consistently overconfident}
if $\mathbb{E}[s|v] > P(y = \hat{y} |v), \forall v$, i.e., the expected value of the model's scores $f$ as a function of $v$ is always greater than the true accuracy as a function of $v$.
\end{definition}
Consistent underconfidence can be defined analogously, using $\mathbb{E}[s|v] < P(y = \hat{y} |v), \forall v$.
In the special case where the variable $V$ is defined as the score itself, we have the condition $ s > P(y = \hat{y} |s), \forall s$, leading to consistent overconfidence for the scores.

For the case of consistent over- or under- confidence for a model $f$, we have the following result:
\begin{theorem}[Equality conditions of ECE and VECE]
Let $f$ be a classifier that is consistently under- or over- confident with respect both to $s$ and to a variable $V$. Then the ECE and $\text{VECE}$ of $f$ are equal.
\end{theorem}

The results above provide insight into the relationship between ECE and VECE. Specifically, if the miscalibration is ``one-sided" (i.e., consistently over- or under-confident for both the score $s$ and a variable $V$) then ECE and VECE will be in agreement. However, when the classifier $f$ is both over- and under-confident (as a function of either $s$ or $v$), then ECE and VECE can differ significantly and, as a result, ECE can mask significant systematic miscalibration with respect to variables of interest.


\begin{figure*}[!ht]
    \centering
    \includegraphics[width=0.83\linewidth]{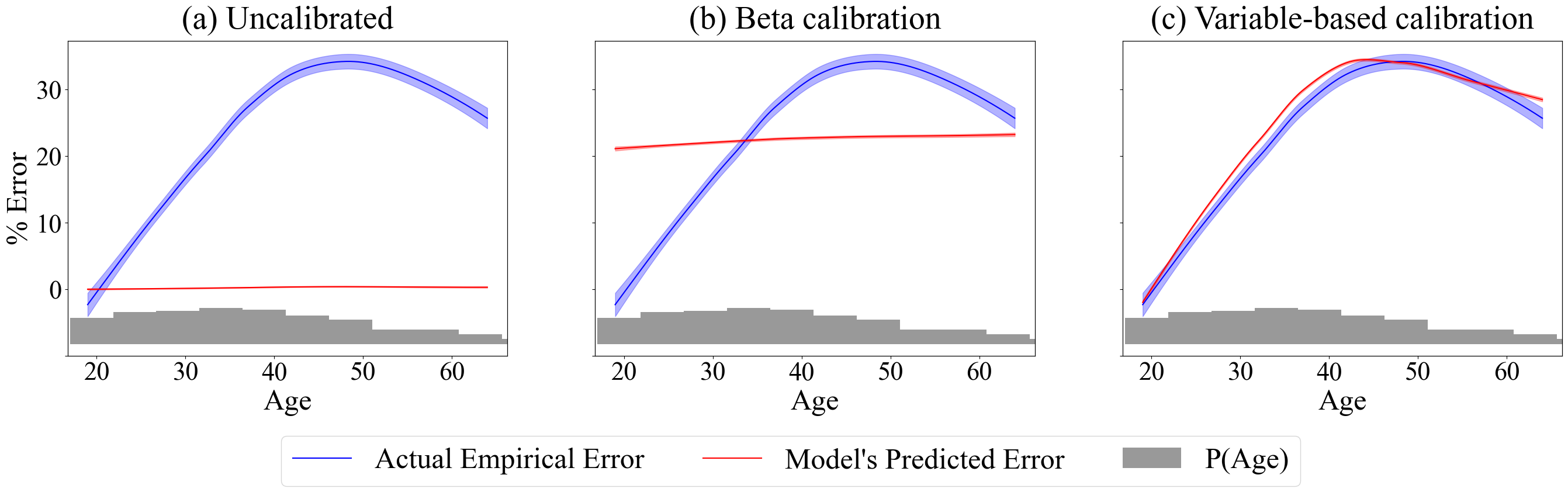} 
    \caption{Variable-based calibration plots for \textit{Age} for the Adult Income model} \label{fig:adult_co}
\end{figure*}

\section{Mitigating Variable-Based Miscalibration}
\label{section:experiments}
\subsection{Diagnosis of Variable-Based Miscalibration}
\label{section:detection}
In order to better detect and characterize per-variable miscalibration, we discuss below \textit{variable-based calibration plots}, which we have found useful in practice. Figure \ref{fig:introb} shows an example of a variable-based calibration plot for age. In Section \ref{section:diffeqs}, we explore how these plots can be used to characterize miscalibration across different classifiers, datasets, and variables of interest.

For ease of interpretation in the results below we focus on the model's error rate and predicted error, rather than accuracy and confidence, although they are equivalent. Particularly for models with high accuracy,  we find that it is more intuitive to discuss differences in error rate than in accuracy. 

To generate these types of plots, we first compute the individual error $\mathbb{I}[y \neq \hat{y}]$ and predicted error $1-s(x) = 1 - \text{max}(f(x))$ for each observation. We then construct nonparametric error curves with LOESS. (Further details are available in Appendix B.) This approach allows us to obtain 95\% confidence bars for the error rate and mean predicted error, based on standard error, thus putting the differences in curves into perspective.

Beyond visualization, we can use VECE scores to discover which variables for a dataset have the largest systematic variable-based miscalibration. In particular, ranking features in order of decreasing VECE highlights variables that may be worth investigating. An example of such a ranking for the Adult Income dataset\footnote{https://archive.ics.uci.edu/ml/datasets/adult}, based on a neural network with post-hoc beta calibration \citep{kull2017}, is shown in Table \ref{table:adultrank}. The {\it Years of education} and {\it Age} variables rank highest in VECE, so a model developer or a user of a model might find it useful to generate a variable-based calibration plot for each of these. The {\it Weekly work hours} and {\it Census weight} variables are of lesser concern, but could also be explored. We will perform an in-depth investigation of miscalibration with respect to the variable {\it Age} in Section \ref{section:adultexp}.

\begin{table}[!ht]
\centering
\begin{tabular}{lll}
           & \textbf{VECE} & \textbf{VCE($v^*$)}  \\
 \hline\hline
Years of education & 9.95\%    &  20.13\%\\
Age & 9.59\%  & 23.44\% \\
Weekly work hours & 7.94\% & 18.21\% \\
Census weight & 5.06\% & 12.08\% \\

\end{tabular}
\caption{Variable-based calibration error of Adult Income dataset features}
\label{table:adultrank}
\end{table}

It is also possible to define the maximum value of $\text{VCE}(v)$, i.e, the {\it worst-case calibration error}, as well as the value $v^*$ that incurs this worst-case error:
 \begin{align}
     v^* &= \arg\max_{v} \{ \text{VCE}(v) \} \\
     &= \arg\max_{v} \{ \bigl| P(y=\hat{y}|v) - \mathbb{E}[s(v)] \bigr| \} \nonumber
 \end{align}
 
Estimating either $v^*$ or $\text{VCE}(v^*)$ accurately may be difficult in practice, particularly for small sample sizes $n$, since it involves the non-parametric estimation of the difference of two curves \citep{bowman1996graphical} as a function of $v$ (as the shapes of the curves need not follow any convenient parametric form, e.g., see Figure \ref{fig:introb}). One simple estimation strategy is to smooth both curves with LOESS and compute the maximum difference between the two estimated curves. Using this approach, worst-case calibration errors $\text{VCE}(v^*)$ for the Adult Income model are also shown in Table \ref{table:adultrank}.

\subsection{Calibration Methods}
\label{sec:recal_methods}
We found empirically, across multiple datasets, that standard score-based calibration methods often reduce ECE while neglecting variable-based systematic miscalibration. Because calibration error can vary as a function of a feature of interest $V$, we propose incorporating information about $V$ into post-hoc calibration. In particular, we introduce the concept of \textit{variable-based calibration methods}, a family of calibration methods that adjust confidence scores with respect to some variable of interest $V$. As an illustrative example, we perform experiments in Section \ref{section:diffeqs} with a modification of probability calibration trees \citep{leatharttrees}. This technique involves performing logistic calibration separately for data splits defined by decision trees trained over the input space. We alter the method to train decision trees for $y$ with only $v$ as input, with a minimum leaf size of one-tenth of the total calibration set size. We then perform beta calibration at each leaf \citep{kull2017}, as we found in our experiments that it performs empirically better than logistic calibration. In the multi-class case, we use Dirichlet calibration, an extension of beta calibration for $k$-class classification \citep{kull2019}. Our split-based calibration method using decision trees is intended to provide a straightforward illustration of the potential benefits of variable-based calibration, rather than a state-of-the-art methodology that can balance ECE and VECE (which we leave to future work). We also investigated variable-based calibration methods that operate continuously over $V$ (rather than on separate data splits) using extensions of logistic and beta calibration, but found that these were not as reliable in our experiments as the tree-based approach (see Appendix C for details).

\section{Variable-Based Miscalibration in Practice}
\label{section:diffeqs}
In this section, we explore several examples where the ECE obscures systematic miscalibration relative to some variable of interest, particularly after post-hoc score-based calibration. In our experiments we use four datasets that span tabular, text, and image data.
For each dataset and variable of interest $V$, we investigate both (1) several score-based calibration methods and (2) our variable-based calibration method (the tree-based technique described in Section \ref{sec:recal_methods}), comparing the resulting ECE, VECE, and variable-based calibration plots. In particular, we calibrate with scaling-binning \citep{kumar2019verified}, Platt scaling \citep{plattprobabilistic}, beta calibration \citep{kull2017}, and, for the multi-class case, Dirichlet calibration \citep{kull2019}.
The datasets are split into training, calibration, and test sets. Each calibration method is trained on the same calibration set, and all metrics and figures are produced from the final test set. The ECE and VECE are computed with an equal-support binning scheme, with $B=10$. Further details regarding datasets, models, and calibration methods are in Appendix B.


\begin{figure*}[!ht]
    \centering
    \includegraphics[width=0.83\linewidth]{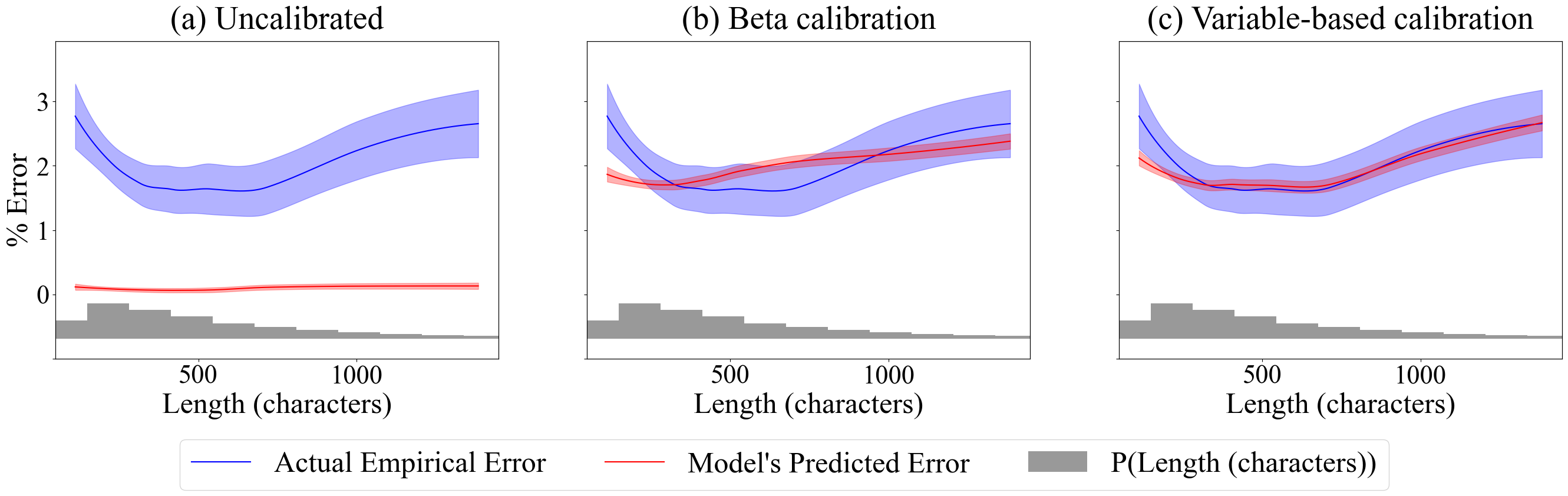} 
    \caption{Variable-based calibration plots for the Yelp model for \textit{Review Length}} \label{fig:yelp_co}
\end{figure*}

\subsection{Adult Census Records: Predicting Income}
\label{section:adultexp}
The Adult Income dataset
consists of 1994 Census records; the goal is to predict whether an individual's annual income is greater than \$50,000. We model this data with a simple feed-forward neural network and evaluate the model's calibration error with respect to age (i.e. let $V$=age). Uncalibrated, this model has an ECE and VECE of 20.67\% (see Table \ref{table:adult}). The ECE and VECE are equal precisely because of the model's consistent overconfidence as a function of both the confidence score and $V$ (see Definition \ref{sec:overconf}). This overconfidence with respect to age is reflected in the variable-based calibration plot (Figure \ref{fig:adult_co}a). The model's error rate varies significantly as a function of age, with very high error for individuals around age 50, and much lower error for younger and older people. However, its confidence remains nearly constant at close to 100\% (i.e., a predicted error close to 0\%) across all ages. 

\begin{table}[!ht]
\centering
\begin{tabular}{lll}
           & \textbf{ECE} & \textbf{VECE}  \\
 \hline\hline
Uncalibrated & 20.67\%  & 20.67\%       \\
\hline
Scaling-binning & 2.27\%  & 9.25\% \\
Platt scaling & 4.57\%  & 10.13\%  \\
Beta calibration & 1.65\%  & 9.59\%  \\
\hline
Variable-based calibration & \textbf{1.64\%}  & \textbf{2.11\%}  \\
\end{tabular}
\caption{Adult Income model calibration error}
\label{table:adult}
\end{table}


\begin{figure*}[!ht]
    \centering
    \includegraphics[width=0.83\linewidth]{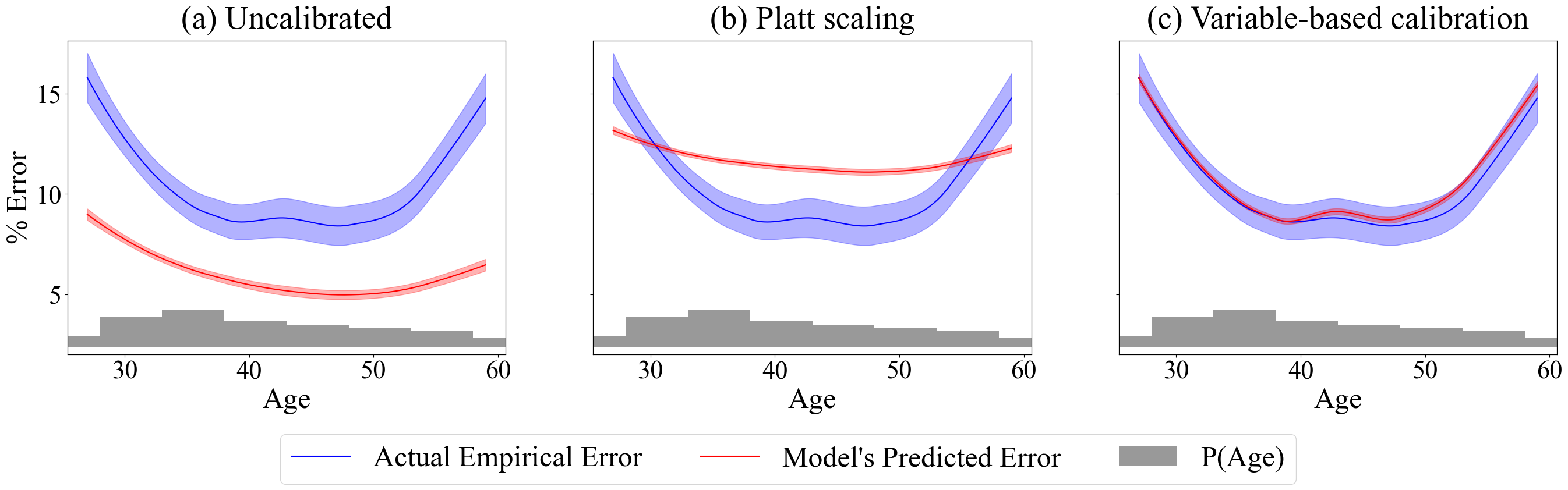} 
    \caption{Variable-based calibration plots for the Bank Marketing model for \textit{Age}} \label{fig:bank_co}
\end{figure*}

After calibrating, the ECE is dramatically reduced, with beta calibration achieving an ECE of 1.65\%. However, the corresponding VECE is still very high (over 9\%). As shown in Figure \ref{fig:adult_co}b, the model's self-predicted error has increased substantially, but remains near constant as a function of age. Thus, despite a significant improvement in ECE, the model still harbors unfairness with respect to age, exhibiting overconfidence in its predictions for individuals in the 35-65 age range, and underconfidence for those outside of it. As the model is no longer consistently overconfident, the ECE and VECE diverge, as predicted theoretically.

Variable-based calibration obtains a significantly lower VECE of 2.11\%, while simultaneously reducing the ECE. This improvement in VECE is reflected in Figure \ref{fig:adult_co}c. The model's predicted error now varies with age to match the true error rate. In this case, a simple variable-based calibration method improves the age-wise systematic miscalibration of the model, without detriment to the overall calibration error.

\subsection{Yelp Reviews: Predicting Sentiment}
To explore variable-based calibration in an NLP context, we use a fine-tuned large language model, BERT \citep{devlin2018}, on the Yelp review dataset\footnote{https://www.yelp.com/dataset}.
The model predicts whether a review has a positive or negative rating based on its text. In this case there are no easily-interpretable features directly input to the model. Instead, to better diagnose model behavior, we can analyze real-valued characteristics of the text, such as the length of each review or part-of-speech statistics. Here we focus on review length in characters. 

Figure \ref{fig:yelp_co} shows the model's error and predicted error with respect to review length. The error rate is lowest for reviews around 300-700 characters, around the median review length. Very short and very long reviews are associated with a higher error rate. Uncalibrated, this model is consistently overconfident, with an ECE and VECE of 1.93\% (see Table \ref{table:yelp}).

\begin{table}[!ht]
\centering
\begin{tabular}{lll}
           & \textbf{ECE} & \textbf{VECE}  \\
 \hline\hline
Uncalibrated & 1.93\%  & 1.93\%       \\
\hline
Scaling-binning & 4.23\%  & 4.23\% \\
Platt scaling & 3.04\%  & 0.64\%  \\
Beta calibration & 1.73\%  & 0.37\%  \\
\hline
Variable-based calibration & \textbf{1.70\%}  & \textbf{0.23\%}  \\
\end{tabular}
\caption{Yelp model calibration error}
\label{table:yelp}
\end{table}

After beta calibration, the ECE and VECE drop to 1.73\% and 0.37\%, respectively. Figure \ref{fig:yelp_co}b reflects this: the model's predicted error aligns more closely with its actual error rate, although it is still overconfident for very short reviews.

Our variable-based calibration method further reduces the VECE and yields a small improvement to the ECE. The new predicted error curve matches the true relationship between review length and error rate more faithfully (Figure \ref{fig:yelp_co}c), reducing overconfidence for short reviews.


\begin{figure*}[!ht]
    \centering
    \includegraphics[width=0.83\linewidth]{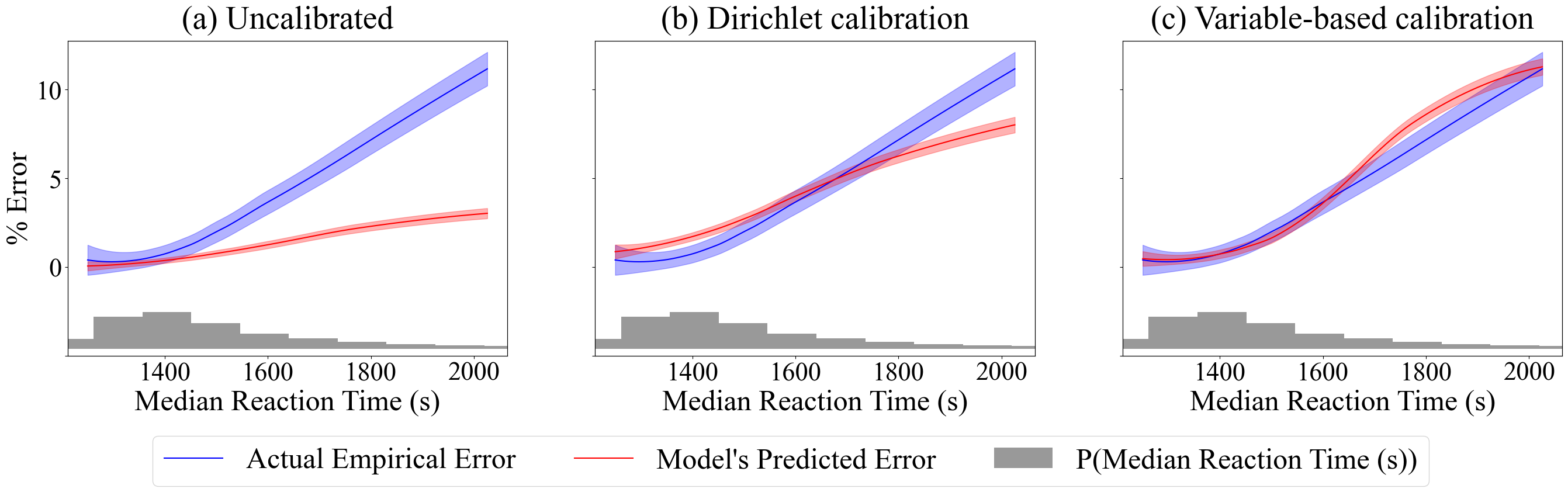} 
    \caption{Variable-based calibration plots for the CIFAR-10H model for \textit{Median Reaction Time}} \label{fig:cifar_co}
\end{figure*}

\subsection{Bank Marketing: Predicting
Subscriptions}
We also investigate miscalibration on a simple neural network modeling the Bank Marketing dataset\footnote{https://archive.ics.uci.edu/ml/datasets/bank+marketing}.
The model predicts whether a bank customer will subscribe to a bank term deposit as a result of direct marketing. Uncalibrated, the model is overconfident, with both ECE and VECE over 4.5\% (see Table \ref{table:bank}). Consider the calibration error with respect to customer age before (Figure \ref{fig:bank_co}a) and after (Figure \ref{fig:bank_co}b) Platt scaling. Platt scaling, which is the best-performing score-based calibration method, uniformly increases the predicted error across age, reducing both ECE and VECE, but resulting in underconfidence for most ages and overconfidence at the edges of the distribution.

\begin{table}[!ht]
\centering
\begin{tabular}{lll}
           & \textbf{ECE} & \textbf{VECE}  \\
\hline\hline
Uncalibrated & 4.69\%  & 4.69\%       \\ 
\hline
Scaling-binning & 4.37\%  & 3.39\% \\
Platt scaling & 2.38\%  & 2.83\%  \\
Beta calibration & 2.48\%  & 2.77\%  \\
\hline
Variable-based calibration & \textbf{2.10}\%  & \textbf{0.52\%}  \\
\end{tabular}
\caption{Bank Marketing model calibration error}
\label{table:bank}
\end{table}

The variable-based calibration method achieves competitive ECE, while reducing VECE to about half of one percent. The calibration plot reflects this improvement: the predicted error matches the true error rate more closely, reducing the miscalibration with respect to customer age.

\subsection{CIFAR-10H: Image Classification}

As a multi-class example, we investigate variable-based miscalibration on CIFAR-10H, a 10-class image dataset including labels and reaction times from human annotators \citep{peterson2019human}. We use a standard deep learning image classification architecture (a DenseNet model) to predict the image category, and investigate median annotator reaction times, metadata that are not provided to the model. Instead of Platt scaling and beta calibration, here we use Dirichlet calibration (to accomodate the multiple classes). 

In this case, Dirichlet calibration achieves the lowest overall ECE and variable-based calibration obtains the lowest VECE (see Table \ref{table:cifar}). The variable-based calibration plots are shown in Figure \ref{fig:cifar_co}. We see that the variable-based calibration method reduces underconfidence for examples with low median reaction times (where the majority of data points lie). 

\begin{table}[!ht]
\centering
\begin{tabular}{lll}
           & \textbf{ECE} & \textbf{VECE}  \\
\hline\hline
Uncalibrated & 1.90\%  & 1.92\%       \\ 
\hline
Scaling-binning & 3.83\%  & 3.60\% \\
Dirichlet calibration & \textbf{0.80\%}  & 1.12\%  \\
\hline
Variable-based calibration & 1.18\%  & \textbf{0.86\%}  \\
\end{tabular}
\caption{CIFAR-10H model calibration error}
\label{table:cifar}
\end{table}

\paragraph{Summary of Experimental Results}
Our results demonstrate the potential of variable-based calibration. While score-based calibration methods generally improved the ECE, variable-based calibration methods performed better across datasets in terms of simultaneously reducing both the ECE and VECE, without any significant increase in model error rate or the VECE for other variables (details in Appendix B). The results also illustrate that variable-based calibration plots enable meaningful characterization of the relationships between variables of interest and predicted/true error, providing more detailed insight into a model's performance than a single number (i.e., ECE or VECE). 

\section{Discussion and Conclusions}
\label{section:conclusion}
\paragraph{Discussion of Limitations} There are several potential limitations of this work. First, we focused on the mitigation of miscalibration for one variable $V$ at a time. 
Although we did not observe higher VECE for other variables after applying our variable-based calibration method, this behavior has not been analyzed theoretically. Further, a more thorough investigation of miscalibration across intersections of variables is still needed. We also emphasize that the variable-based calibration method used in the paper is primarily for illustration; the development of new methods for simultaneously reducing score-based and variable-based miscalibration is a useful direction for future work.

\paragraph{Conclusions} In this paper we demonstrated theoretically and empirically that ECE can obscure significant miscalibration with respect to variables of potential importance to a developer or user of a classification model. To better detect and characterize this type of miscalibration, we introduced the VECE measure and corresponding variable-based calibration plots, and we characterized the theoretical relationship between VECE and ECE. In a case study across several datasets and models, we showed that VECE, variable-based calibration plots, and variable-based calibration methods are all useful tools for understanding and mitigating miscalibration on a per-variable level. Looking forward, to mitigate biases in calibration error, we recommend moving beyond purely score-based calibration analysis. In addition to promoting fairness, these techniques offer new insight into model behavior and provide actionable avenues for improvement.

\section*{Acknowledgements}
This material is based upon work supported in part by the HPI Research Center in Machine Learning and Data Science at UC Irvine, by the National Science Foundation under grants number 1900644 and 1927245, and by a Qualcomm Faculty Award.

\bibliography{refs}

\onecolumn
\appendix
\include{appendix}

\end{document}

%% file: appendix.tex
\onecolumn

\section{Proofs for Section 5}
\label{section:proofs}

\begin{theorem}[VECE bound]
There exist $K$-ary classifiers $f$ and variables $V$ such that the classifier $f$ has ECE = 0 and
$ \text{VECE} = 0.5 - \frac{1}{2K}$.
\end{theorem}

\begin{proof}

Let $V$ be a continuous variable with density $P(v)$.
Recall that $\text{VECE} = \int_v P(v) | P(y = \hat{y}|v) - \mathbb{E}[s|v] |dv$ where $P(y = \hat{y}|v)$ is the accuracy of model $f$ as a function of $v$, and the score $s$ is the probability that the model assigns to its label prediction $\hat{y}$.

The reliability diagram for a $K$-ary classifier has scores $s \in [ \frac{1}{K}, 1]$ where the leftmost value for this interval is a result of the fact that the score is defined as the maximum of $K$ class probabilities. Let $\gamma = 0.5 + \frac{1}{2K}$ be the midpoint of this interval. 

Assume that the scores $s$ have a uniform distribution of the form $s \sim U(\gamma-\alpha, \gamma + \alpha)$, where $\alpha$ is some constant and $0 \le \alpha \le 0.25$, and that the scores $s$ and the variable $V$ are independent. 

Further assume that  the accuracy of the model $f$ depends on  $v$ and $s$ in the following manner
\begin{align*}
    & P(y=\hat{y} | v \le v_t, s \le \gamma) = 1 - \alpha && P(y=\hat{y} | v \le v_t, s> \gamma) = 1\\
    & P(y=\hat{y} | v > v_t, s \le \gamma) = \frac{1}{K} && P(y=\hat{y} | v > v_t, s>\gamma) =\frac{1}{K} + \alpha
\end{align*} 
where $v_t$ is defined such that $P(v\le v_t) = P(v > v_t) = 0.5$.

The marginal accuracy as a function of the score (marginalizing over $v$) can be written as 
\begin{align*}
    & P(y=\hat{y} |   s \le \gamma) = \gamma - \frac{\alpha}{2}   \\
    & P(y=\hat{y} |   s > \gamma) = \gamma + \frac{\alpha}{2}.  
\end{align*} 

The marginal accuracy as a function of $v$ (marginalizing over $s$) is
\begin{align*}
    & P(y=\hat{y} | v \le v_t   ) = 1 - \frac{\alpha}{2}   \\
    & P(y=\hat{y} |   v > v_t   ) = \frac{1}{K} + \frac{\alpha}{2}.  
\end{align*} 

This setup is designed so that the score is close to the accuracy as a function of $s$ (to minimize ECE), but the variable-based expected scores $\mathbb{E}[s|v]  = \gamma$ are relatively far away from accuracy as a function of $v$.

Under these assumptions we can write the ECE as
\begin{align}
\begin{split}
\text{ECE} &= \int_{s} p(s) \cdot |P(y=\hat{y}|s) - s| ds \\
&= \int_{\gamma-\alpha}^{\gamma} \frac{1}{2 \alpha} | \gamma - \frac{\alpha}{2} - s | ds +\int_{\gamma}^{\gamma+\alpha} \frac{1}{2 \alpha} | \gamma + \frac{\alpha}{2} - s | ds \\
&= \frac{\alpha}{4}.
\end{split}
\end{align}
 
We can write the VECE as
\begin{align}
\begin{split}
\text{VECE} &=   \int_{-\infty}^{v_t} P(v) \cdot |P(y=\hat{y}|v) - \mathbb{E}[s|v]| dv
 +  \int_{v_t}^{\infty} P(v) \cdot |P(y=\hat{y}|v) - \mathbb{E}[s|v]| dv \\
& = \int_{-\infty}^{v_t} P(v) \cdot | 1 - \frac{\alpha}{2}- \gamma | dv
 +  \int_{v_t}^{\infty} P(v) \cdot | \frac{1}{K} + \frac{\alpha}{2}  - \gamma|  dv \\
 & =  (0.5 - \frac{1}{2K} - \frac{\alpha}{2})  \int_v P(v) dv  \\
 & = 0.5 - \frac{1}{2K} - \frac{\alpha}{2}.
\end{split}
\end{align}

Thus, as $\alpha \to 0$, $\text{VECE} \to (0.5 - \frac{1}{2K})$ and $\text{ECE} \to 0$.

\end{proof}

\begin{theorem}[ECE bound]
There exist K-ary classifiers $f$ and variables $V$ such that the classifier $f$ has $\text{VECE}=0$ and $\text{ECE}=0.5 - \frac{1}{2K}$.
\end{theorem}

\begin{proof}
Let $V$ be a continuous variable with density $P(V)$. Recall that a K-ary classifier has scores $s \in [\frac{1}{K},1]$, where we let $\gamma = 0.5 + \frac{1}{2K}$ be the midpoint of this interval. Assume that $f$ produces scores from two uniform distributions, with equal probability: $s \sim U(\frac{1}{K}, \frac{1}{K} + \alpha)$ and $s \sim U(1-\alpha,1)$, where $\alpha$ is some constant $0 \le \alpha \le 0.25$, and that the scores $s$ and the variable $V$ are independent. Finally, suppose the accuracy of the model $P(y=\hat{y}) = \gamma$ is independent of $s$ and $V$.

Under these assumptions we can write the VECE as

\begin{align}
\begin{split}
\text{VECE} &=   \int_{-\infty}^{\infty} P(v) \cdot |P(y=\hat{y}|v) - \mathbb{E}[s|v]| dv \\
& = \int_{-\infty}^{\infty} P(v) \cdot |\gamma - \gamma| dv \\
 & = 0.
\end{split}
\end{align}

We can write the ECE as

\begin{align}
\begin{split}
\text{ECE} &= \int_{s} p(s) \cdot |P(y=\hat{y}|s) - s| ds \\
&= \frac{1}{2} \int_{\frac{1}{K}}^{\frac{1}{K}+\alpha} \frac{1}{\alpha} | \gamma - s | ds + \frac{1}{2} \int_{1-\alpha}^{1} \frac{1}{\alpha} | \gamma - s | ds \\
&=0.5 - \frac{1}{2K} - \frac{\alpha}{2}.
\end{split}
\end{align}

Thus, as $\alpha \to 0$, $\text{ECE} \to (0.5 - \frac{1}{2K})$ and $\text{VECE} = 0$.

\end{proof}

\begin{definition}[Consistent overconfidence]
Let $f$ be a classifier with scores $s$. For a variable $V$ taking values $v$,  $f$ is {\it consistently overconfident}
if $\mathbb{E}[s|v] > P(y = \hat{y} |v), \forall v$, i.e., the expected value of the model's scores $f$ as a function of $v$ is always greater than the true accuracy as a function of $v$.
\end{definition}
Consistent underconfidence is defined analogously with $\mathbb{E}[s|v] < P(y = \hat{y} |v), \forall v$.
In the special case where the variable $V$ is defined as the score itself, we have $ s > P(y = \hat{y} |s), \forall s$, etc.

\begin{theorem}[Equality conditions for ECE and VECE]
Let $f$ be a classifier that is consistently under- or over-confident with respect both to $s$ and to a variable $V$. Then the ECE and VECE of $f$ are equal.
\end{theorem}

\begin{proof}
Without loss of generality, suppose $f$ is consistently underconfident with respect to its scores $s$ and $V$.

Then we have, by consistent underconfidence:
\begin{equation}
  \begin{split}
    \text{ECE} &= \int_{s} p(s) \cdot |P(y=\hat{y}|s) - s| ds\\
    &= \int_{s} p(s) \cdot P(y=\hat{y}|s) ds - \mathbb{E}[s]\\
  \end{split}
  \qquad
  \begin{split}
    \text{VECE} &= \int_{v} p(v) \cdot |P(y=\hat{y}|v) - \mathbb{E}[s|v]| dv\\
    &= \int_{v} p(v) \cdot P(y=\hat{y}|v) dv - \int_v p(v) \mathbb{E}[s|v] dv \cdot \\
     &= \int_{v} p(v) \cdot P(y=\hat{y}|v) dv - \mathbb{E}[s]\\
  \end{split}
\end{equation}

By the law of total probability,
\begin{equation}
  \begin{split}
    &= \int_{s} p(s) \cdot P(y=\hat{y}|s) ds - \mathbb{E}[s]\\
    &= P(y=\hat{y}) - \mathbb{E}[s]\\
  \end{split}
  \qquad\qquad
  \begin{split}
    &= \int_{v} p(v) \cdot P(y=\hat{y}|v) dv - \mathbb{E}[s]\\
    &= P(y=\hat{y}) - \mathbb{E}[s]\\
  \end{split}
\end{equation}
So $\text{ECE} = \text{VECE} = P(y=\hat{y}) - \mathbb{E}[s]$.

\end{proof}

\clearpage

\section{Calibration, Model, and Dataset Details}
\label{sec:appdxa}
Here, we include additional information and plots for each dataset and model discussed in Section 7. Code for reproducing all tables and plots is available online.\footnote{\raggedright{Our code is available online at https://github.com/markellekelly/variable-wise-calibration.}} 

On each dataset, we test several existing calibration methods: Platt scaling, scaling-binning, beta calibration, and (for the multi-class case) Dirichlet calibration. For scaling-binning, we calibrate over 10 bins, and for Dirichlet calibration, we use a lambda value of 1e-3, values chosen based on the respective authors' provided examples. Here and in Section 7, we present the uncalibrated and variable-based calibrated output, along with the best-performing score-based calibration method (for the Adult and Yelp datasets, beta calibration; for Bank Marketing, Platt scaling; for CIFAR, Dirichlet calibration).

Our variable-based calibration method is performed as follows. Given the calibration set, a decision tree classifier is trained to predict the outcome $y$ with input $V$ (the single variable of interest). We use a maximum depth of two and a minimum leaf size of $0.1 *$ the size of the calibration set. The calibration set is then split according to the leaf nodes of the trained decision tree, and separately the rest of the dataset is split according to the same rules. Standard beta calibration is then performed separately for each split, using the subset of the original calibration set as the new calibration set, and computing the new calibrated probabilities for the subset of the original dataset.

Variable-based calibration plots are created with LOESS, with quadratic local fit and an assumed symmetric distribution of the errors, with empirically-chosen smoothing factors between 0.8 and 0.9.

We note the VECE for each numeric variable in each dataset before and after the calibration method is applied. We find in general empirically that variable-based calibration with respect to one variable is not detrimental to the VECE of other variables. 

Finally, we observe that our variable-based calibration method does not tend to significantly degrade accuracy. Accuracies for each dataset before and after its application are shown in Table \ref{table:accuracies}.

\begin{table}[!ht]
\centering
\begin{tabular}{lllll}
           & \textbf{Adult Income} & \textbf{Yelp} & \textbf{Bank Marketing} & \textbf{CIFAR}  \\
\hline
Uncalibrated & 79.1\%  & 98.0\%   & 88.9\%   & 97.2\%  \\ 
Score-based calibrated & 79.1\%  & 98.0\%   & 88.7\% & 96.9\% \\ 
Variable-based calibrated & 79.1\% & 98.0\%  & 88.7\% &  96.0\%\\ 
\end{tabular}
\caption{Accuracies for all four datasets before calibration, after the highest-ECE score-based calibration (as reported in the main paper and below), and after variable-based calibration.}
\label{table:accuracies}
\end{table}

\subsection{Adult Income}
The Adult Income dataset was modeled with a multi-layer perceptron, with two hidden layers of sizes 100 and 75. Of the 48,842 observations, 32,561 were used for training, 2,500 were used for calibration, and 13,781 were used for testing. The dataset includes six continuous variables: age, fnlwgt (the estimated number of people an individual represents), education-num (a number representing the individual's years of education), capital-gain, capital-loss, and hours-per-week (the number of hours per week that an individual works). 

Based on the beta-calibrated model, education-num and age rank the highest in VECE, as shown in Section \ref{section:experiments}. For all six variables, VECE is reduced by applying the variable-based calibration method with respect to age: 

\begin{table}[!ht]
\centering
\begin{tabular}{llll}
         & \textbf{Uncalibrated}  & \textbf{Beta-calibrated} & \textbf{Variable-based calibrated}  \\
 \hline\hline
education-num & 20.67\% & 9.95\% &   \textbf{8.53\%}  \\
age & 20.67\% & 9.59\% & \textbf{2.11\%} \\
hours-per-week & 20.67\% & 7.94\%  &\textbf{6.02\%} \\
fnlwgt & 20.67\% & 5.06\% & \textbf{4.10\%}\\
capital-gain & 20.67\% & 1.50\% & \textbf{1.39\%}\\
capital-loss & 20.67\% & 1.50\% & \textbf{1.39\%}\\

\end{tabular}
\caption{VECE for numeric variables in the Adult Income dataset: uncalibrated, beta-calibrated, and after variable-based calibration \textbf{with respect to age}.}
\end{table}

Uncalibrated, the model's ECE and VECE are 20.67\%. Of the score-based calibration methods tested, beta calibration achieves the lowest ECE of 1.65\%. Relevant reliability diagrams for the uncalibrated, beta-calibrated, and variable-based calibrated models are shown in Figure \ref{fig:adult_all}.

\begin{figure*}[!ht]
    \centering
    \includegraphics[width=0.83\linewidth]{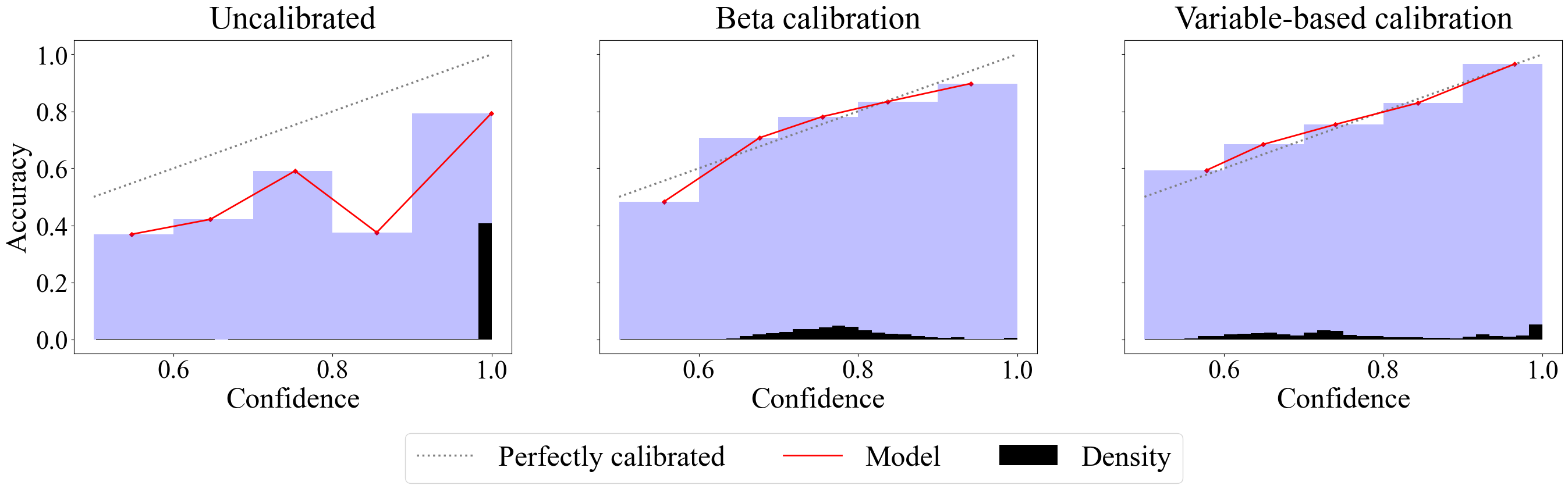} 
    \caption{Reliability diagrams for the Adult Income model} \label{fig:adult_all}
\end{figure*}

\subsection{Yelp}
The Yelp dataset was modeled with a fine-tuned BERT model. 100,000 observations were randomly sampled from the full Yelp dataset. Of these, 70,500 were used for training, 10,000 were used for calibration, and 19,500 were used for testing. Several continuous features were generated from the raw text reviews, including length in characters, number of special characters, and proportions of each part of speech. Based on the beta-calibrated model, review length ranked highest in VECE, followed by proportion of stop words, as shown in Table \ref{tab:yelpvars}.

\begin{table}[!ht]
\centering
\begin{tabular}{llll}
         & \textbf{Uncalibrated}  & \textbf{Beta-calibrated} & \textbf{Variable-based calibrated}  \\
 \hline\hline
Length (characters) & 1.93\% & 0.37\% & \textbf{0.23\%}      \\
Stop-word Proportion &  1.93\% & 0.29\% & \textbf{0.28\%}  \\
Named Entity Count & 1.93\% & \textbf{0.21\%} & 0.22\%  
\end{tabular}
\caption{VECE for numeric variables in the Yelp dataset: uncalibrated, beta calibrated, and after variable-based calibration \textbf{with respect to length in characters}.}
\label{tab:yelpvars}
\end{table}

Uncalibrated, the model's ECE and VECE are 1.93\%. Of the score-based calibration methods tested, beta calibration achieves the lowest ECE of 1.73\%. Relevant reliability diagrams for the uncalibrated, beta-calibrated, and variable-based calibrated models are shown in Figure \ref{fig:yelp_all}.

\begin{figure*}[!ht]
    \centering
    \includegraphics[width=0.83\linewidth]{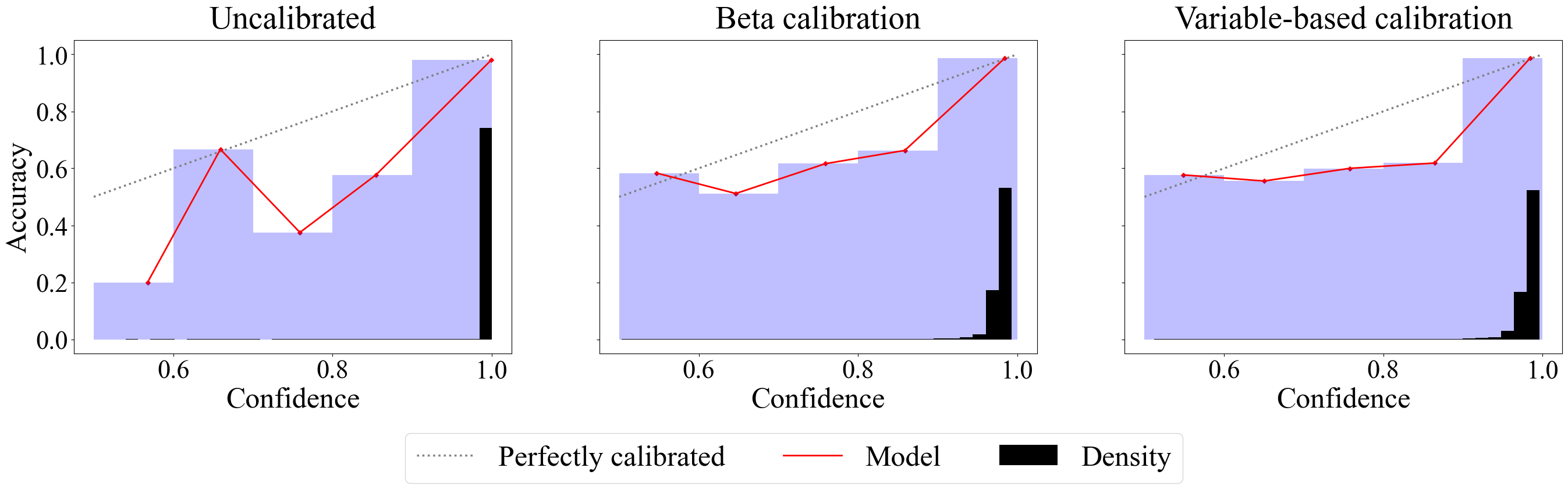} 
    \caption{Reliability diagrams for the Yelp model} \label{fig:yelp_all}
\end{figure*}

\subsection{Bank Marketing}
The Bank Marketing dataset was modeled with a multi-layer perceptron, with two hidden layers of sizes 100 and 75. Of the 45,211 total observations, 31,647 were used for training, 1,000 were used for calibration, and 12,564 were used for testing. Based on the model calibrated with Platt scaling, account balance ranked highest in VECE, followed by age, as shown in Table \ref{tab:bankvars}.

\begin{table}[!ht]
\centering
\begin{tabular}{llll}
         & \textbf{Uncalibrated}  & \textbf{Platt scaling} & \textbf{Variable-based calibrated}  \\
 \hline\hline
Account balance & 5.35\% & 4.17\% &  \textbf{3.22\%}   \\
Age & 4.69\% & 2.83\%  &   \textbf{0.52\%} \\
\end{tabular}
\caption{VECE for numeric variables in the Bank Marketing dataset: uncalibrated, calibrated with Platt scaling, and after variable-based calibration \textbf{with respect to age}.}
\label{tab:bankvars}
\end{table}

Uncalibrated, the model's ECE is 4.69\%. Of the score-based calibration methods tested, Platt scaling achieves the lowest ECE of 2.38\%. Relevant reliability diagrams for the uncalibrated, calibrated with Platt scaling, and variable-based calibrated models are shown in Figure \ref{fig:bank_all}.

\begin{figure*}[!ht]
    \centering
    \includegraphics[width=0.83\linewidth]{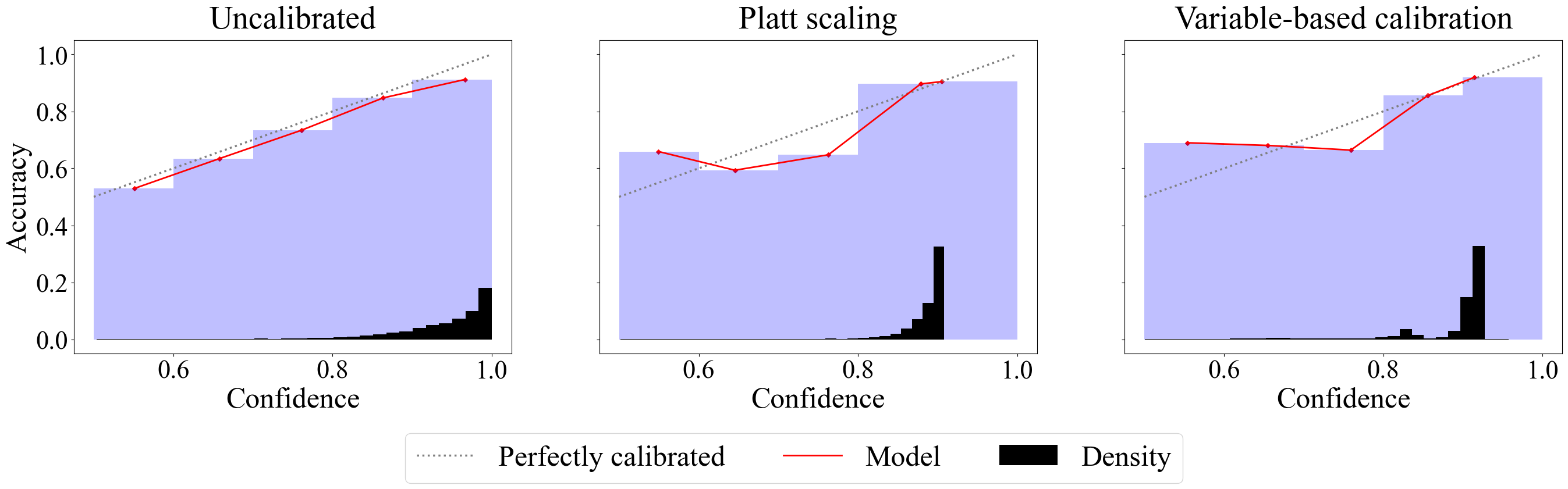} 
    \caption{Reliability diagrams for the Bank Marketing model, uncalibrated (left), calibrated with Platt scaling (middle), and variable-based calibrated (right)} \label{fig:bank_all}
\end{figure*}

\subsection{CIFAR-10H}
The CIFAR-10H dataset was modeled with a DenseNet model. Of the 10,000 total observations, 4,057 were used for training, 2,000 were used for calibration, and 3,943 were used for testing.

Uncalibrated, the model's ECE and VECE are 1.90\% and 1.92\%, respectively. Of the score-based calibration methods tested, Dirichlet calibration achieves the lowest ECE of 0.80\%. Relevant reliability diagrams for the uncalibrated, Dirichlet-calibrated, and variable-based calibrated models are shown in Figure \ref{fig:cifar_all}.

\begin{figure*}[!ht]
    \centering
    \includegraphics[width=0.83\linewidth]{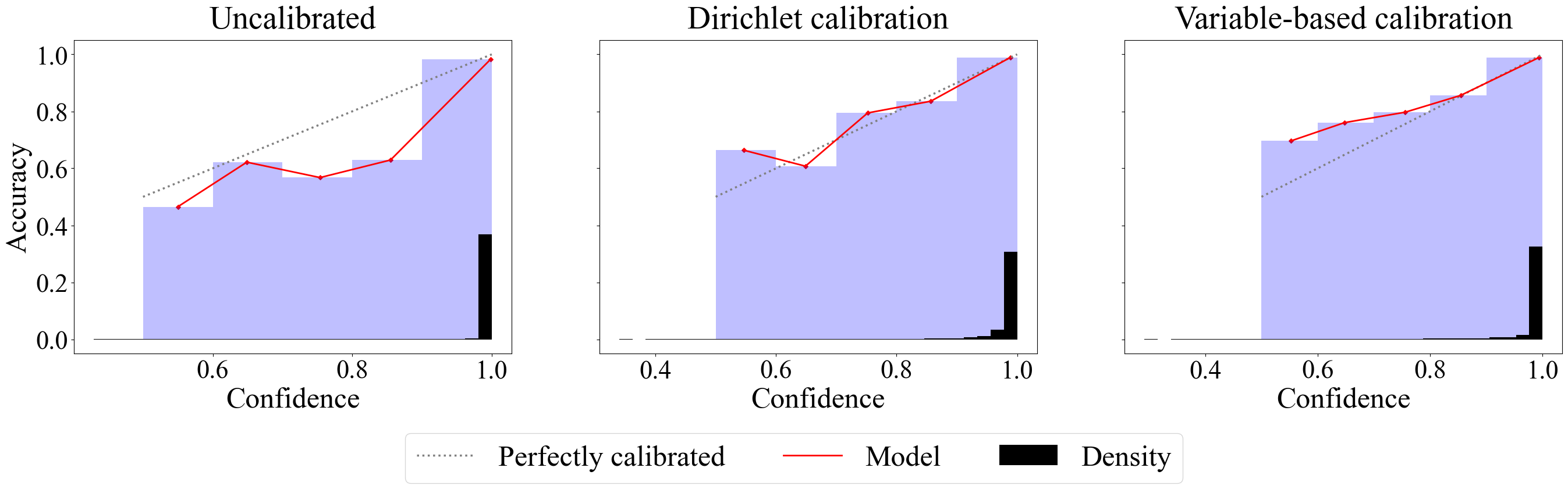} 
    \caption{Reliability diagrams for the CIFAR-10H model, uncalibrated (left), Dirichlet calibrated (middle), and variable-based calibrated (right)} \label{fig:cifar_all}
\end{figure*}

\clearpage
\section{Alternate Calibration Methods}
\label{section:extbetamethod}
As an alternative variable-based calibration method, we extend logistic and beta calibration, which operate continously over score, to incorporate information regarding $V$. In particular, logistic calibration learns a mapping $\mu$ of scores $s$, with parameters $a$ and $c$ learned via logistic regression:

\begin{align*}
\mu_{\text{logistic}}(s) = \frac{1}{1 + 1/\text{exp}(a \cdot s + c)}
\end{align*} 

This can be augmented to include $V$ by simply training the logistic regression on both $s$ and $V$, learning the following mapping:

\begin{align*}
\mu_{\text{logistic}\_v}(s, v) = \frac{1}{1 + 1/\text{exp}(a \cdot s + b \cdot v + c)}
\end{align*}

where $b$ is the logistic regression coefficient corresponding to $V$.

Similarly, beta calibration learns the following mapping, where the parameters $a$, $b$, and $c$ are learned by training a logistic regression on $\text{ln}(s)$ and $-\text{ln}(1-s)$ (see \cite{kull2017} for more details):

\begin{align*}
\mu_{\text{beta}}(s) = \frac{1}{1 + 1/\left( e^c \frac{s^a}{(1-s)^b} \right)}
\end{align*}

This can also be augmented with $V$, including it as a third input to the regression:

\begin{align*}
\mu_{\text{beta}\_v}(s, v) = \frac{1}{1 + 1/\left( e^{d \cdot v + c} \frac{s^a}{(1-s)^b} \right)}
\end{align*}

In contrast to the tree-based method detailed in the main paper, which splits the data along $V$ and then separately calibrates each set, these methods learn one calibration mapping for the entire dataset. Empirically, we find that augmented beta calibration is a promising approach, simultaneously reducing ECE and VECE, although some attention must be paid to the fit of the logistic regression (e.g., by including a quadratic term). However, in our experiments, this technique ultimately was not as reliable as tree-based calibration (perhaps because the functional form of beta calibration is not flexible enough to always be able to correct systematic miscalibration as a function of $V$).

Here, we include the results of the augmented-beta variable-based (VB) calibration method on the Adult Income, Yelp, and Bank Marketing datasets. The models for the Adult and Bank Marketing datasets include a quadratic term for $V$, which obtained a better fit. (Note that this formulation only applies to binary classification, so we do not include results for the CIFAR dataset here).

\begin{table}[!ht]
\centering
\begin{tabular}{lll}
           & \textbf{ECE} & \textbf{VECE}  \\
 \hline
Uncalibrated & 20.67\%  & 20.67\%       \\
Beta calibration & 1.65\%  & 9.59\%  \\
Tree-based VB calibration & 1.64\%  & 2.11\%  \\
Augmented-beta VB calibration & \textbf{1.49\%} & \textbf{1.87\%}\\
\end{tabular}
\caption{Adult Income model calibration error}
\label{table:adult2}
\end{table}

\begin{figure*}[!ht]
    \centering
    \includegraphics[width=0.83\linewidth]{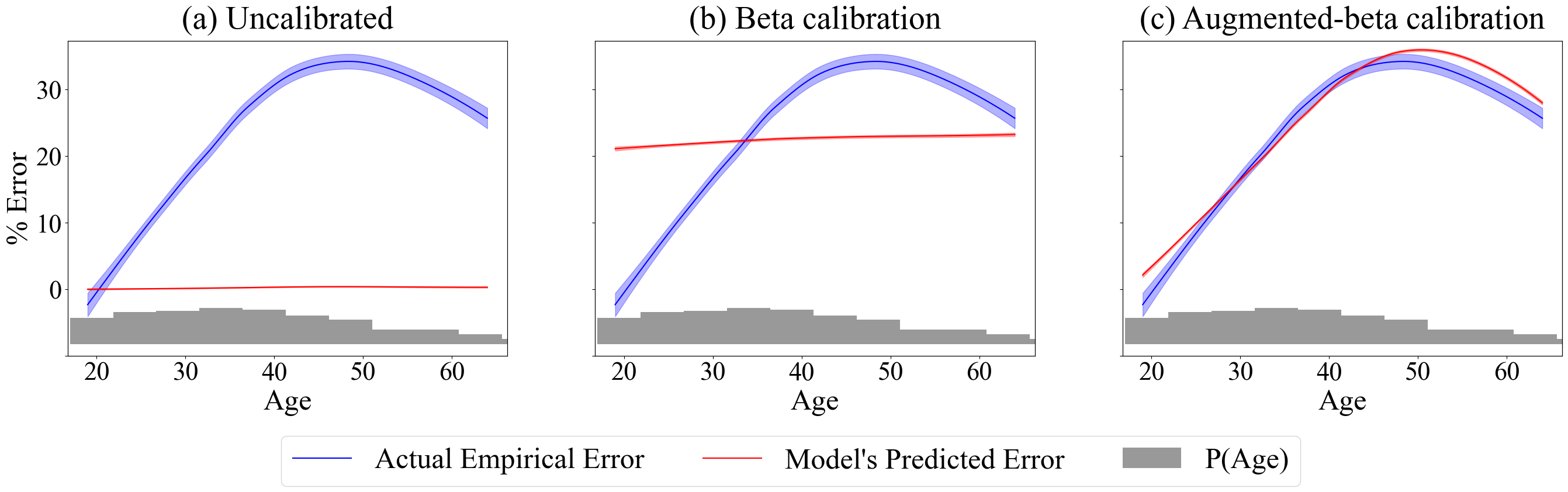} 
    \caption{Variable-based calibration plots for the Adult Income model for \textit{Age}} \label{fig:adult2}
\end{figure*}

\begin{table}[!ht]
\centering
\begin{tabular}{lll}
           & \textbf{ECE} & \textbf{VECE}  \\
 \hline
Uncalibrated & 1.93\%  & 1.93\%       \\
Beta calibration & 1.73\%  & 0.37\%  \\
Tree-based VB calibration & \textbf{1.70\%}  & \textbf{0.23\%}  \\
Augmented-beta VB calibration & 1.73\%  & 0.37\%  \\
\end{tabular}
\caption{Yelp model calibration error}
\label{table:yelp2}
\end{table}

\begin{figure*}[!ht]
    \centering
    \includegraphics[width=0.83\linewidth]{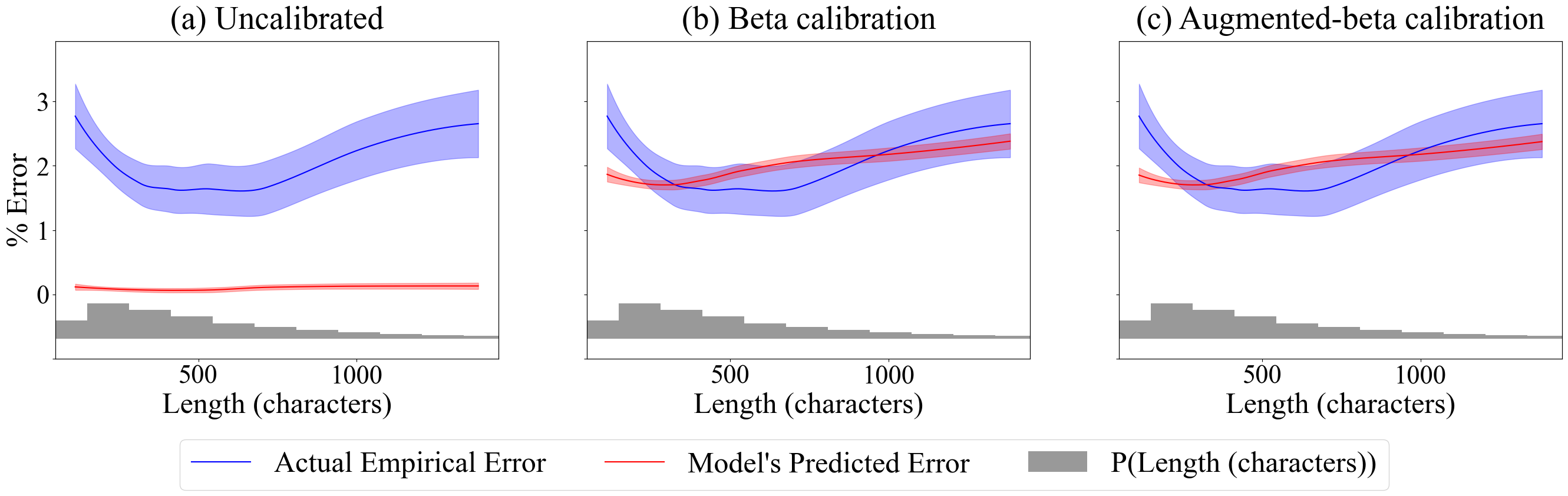} 
    \caption{Variable-based calibration plots for the Yelp model for \textit{Review Length}} \label{fig:yelp2}
\end{figure*}

\begin{table}[!ht]
\centering
\begin{tabular}{lll}
           & \textbf{ECE} & \textbf{VECE}  \\
\hline
Uncalibrated & 4.69\%  & 4.69\%       \\ 
Platt scaling & 2.38\%  & 2.83\%  \\
Tree-based VB calibration & 2.10\%  & \textbf{0.52\%}  \\
Augmented-beta VB calibration & \textbf{2.09}\% & 1.13\%\\
\end{tabular}
\caption{Bank Marketing model calibration error}
\label{table:bank2}
\end{table}

\begin{figure*}[!ht]
    \centering
    \includegraphics[width=0.83\linewidth]{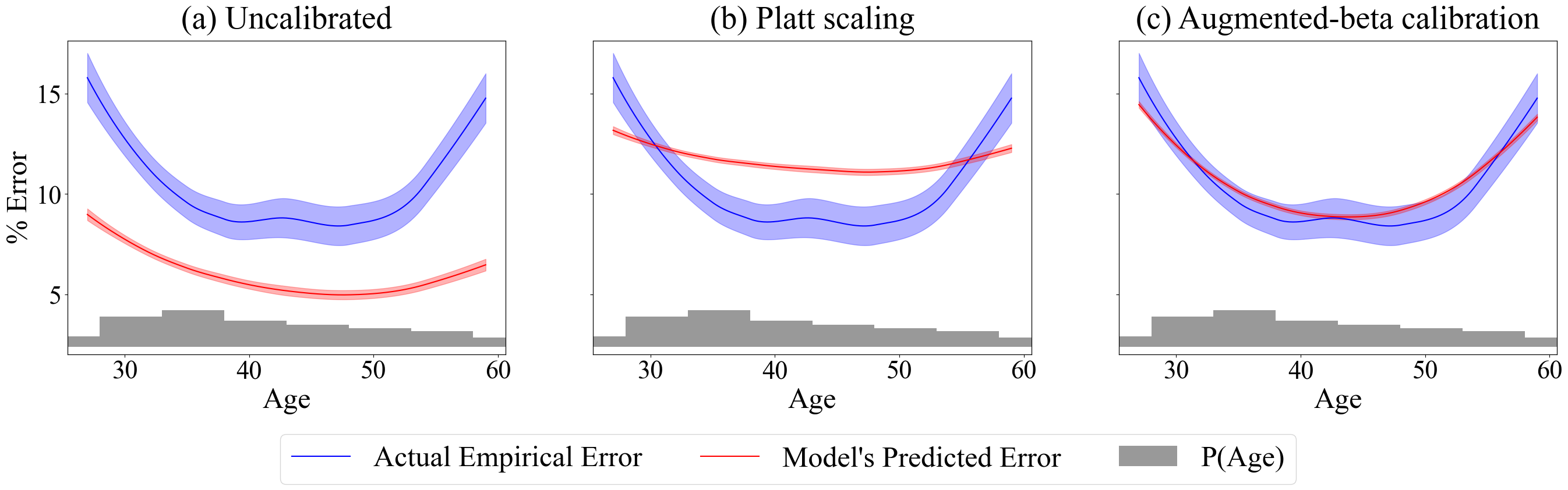} 
    \caption{Variable-based calibration plots for the Bank Marketing model for \textit{Age}} \label{fig:bank2}
\end{figure*}